\documentclass[11pt]{article}

\usepackage[ruled,vlined,linesnumbered]{algorithm2e}
\usepackage[english]{babel}
\usepackage{amsmath, amssymb}
\usepackage{amsthm}
\usepackage{bbm}
\usepackage{graphicx}
\usepackage{latexsym}
\usepackage{graphicx}
\usepackage{cases}
\usepackage{xcolor}
\usepackage{colortbl}
\usepackage{thmtools}
\usepackage{thm-restate}

\declaretheorem[name=Theorem, numberwithin=section]{theorem}
\declaretheorem[name=Lemma, sibling=theorem]{lemma}
\declaretheorem[name=Proposition, sibling=theorem]{proposition}
\declaretheorem[name=Definition, sibling=theorem]{definition}

\declaretheorem[name=Corollary, sibling=theorem]{corollary}

\usepackage{mathtools}
\usepackage{hyperref}
\usepackage{tikz}
\usetikzlibrary{calc, patterns, decorations.pathreplacing, intersections}

\usepackage{fullpage}
\usepackage{authblk}

\usepackage{natbib}
\usepackage{enumitem}

\usepackage{MnSymbol}

\usepackage[capitalize,noabbrev,nameinlink]{cleveref}

\crefformat{equation}{#2(#1)#3}
\Crefformat{equation}{#2(#1)#3}

\Crefname{assumption}{Assumption}{Assumptions}
\Crefname{defn}{Definition}{Definitions}
\Crefformat{assumption}{#2Assumption #1#3}
\crefformat{defn}{#2Definition #1#3}
\Crefformat{defn}{#2Definition #1#3}

\Crefname{lemma}{Lemma}{Lemmas}
\crefname{lemma}{lemma}{lemmas}
\Crefname{proposition}{Proposition}{Propositions}
\crefname{proposition}{proposition}{propositions}
\Crefname{definition}{Definition}{Definitions}
\crefname{definition}{definition}{definitions}

\DeclareMathOperator*{\conv}{conv}

\newcommand{\defeq}{\coloneq}

\newcommand{\poly}{\mathrm{poly}}

\newcommand{\declarecolor}[2]{\definecolor{#1}{RGB}{#2}\expandafter\newcommand\csname #1\endcsname[1]{\textcolor{#1}{##1}}}
\declarecolor{Navy}{0, 0, 128}
\declarecolor{LightBlue}{84, 101, 202}
\declarecolor{Maroon}{128, 0, 0}
\definecolor{mydarkblue}{rgb}{0,0.08,0.45}
\hypersetup{ %
    pdftitle={},
    pdfkeywords={},
    pdfborder=0 0 0,
    pdfpagemode=UseNone,
    colorlinks=true,
    linkcolor=Navy,
    citecolor=Navy,
    filecolor=Navy,
    urlcolor=Navy,
}

\newcommand{\phireg}{\Phi\text{-}\mathsf{Regret}}

\newcommand{\affend}{\mathrm{End}}

\newcommand{\cK}{\mathcal{K}}
\newcommand{\cL}{\mathcal{L}}

\newcommand{\cP}{\mathcal{P}}

\newcommand{\cS}{\mathcal{S}}

\newcommand{\tr}{\mathrm{tr}}

\newcommand{\E}{\mathbb E}
\newcommand{\R}{\mathbb R}

\newcommand{\norm}[1]{\left\|#1\right\|}
\newcommand{\inp}[1]{\left \langle #1 \right \rangle}

\newcommand{\LSR}{\mathsf{LinearSwapReg}}
\newcommand{\PSR}{\mathsf{ProfileSwapReg}}
\newcommand{\PSD}{\mathsf{ProfileSwapDist}}
\newcommand{\PDSR}{\mathsf{PolyDimSwapReg}}

\newcommand{\AL}{\mathsf{AppLoss}}

\newcommand{\bk}{\bar{\kappa}}
\newcommand{\bs}{\bar{s}}

\author[1]{Ioannis Anagnostides}
\author[2]{Gabriele Farina}
\author[2]{Maxwell Fishelson}
\author[3,4]{Haipeng Luo}
\author[4]{Jon Schneider}

\affil[1]{Carnegie Mellon University}
\affil[2]{Massachusetts Institute of Technology}
\affil[3]{University of Southern California}
\affil[4]{Google Research}
\affil[ ]{\small \texttt{ianagnos@cs.cmu.edu}, \texttt{\{gfarina, maxfish\}@mit.edu}, \texttt{haipengl@usc.edu}, \texttt{jschnei@google.com}}

\title{\scalebox{0.95}{Swap Regret Minimization Through Response-Based Approachability}}

\begin{document}

\maketitle
\thispagestyle{empty}

\begin{abstract}
    We consider the problem of minimizing different notions of swap regret in online optimization. These forms of regret are tightly connected to correlated equilibrium concepts in games, and have been more recently shown to guarantee non-manipulability against strategic adversaries. The only computationally efficient algorithm for minimizing \emph{linear swap regret} over a general convex set in $\R^d$ was developed recently by Daskalakis, Farina, Fishelson, Pipis, and Schneider (STOC~'25). However, it incurs a highly suboptimal regret bound of $\Omega(d^4 \sqrt{T})$ and also relies on computationally intensive calls to the ellipsoid algorithm at each iteration. %Subsequent work that extends the scope beyond linear swap regret shares these limitations.

    In this paper, we develop a significantly simpler, computationally efficient algorithm that guarantees $O(d \sqrt{T})$ linear swap regret for a general convex set that has been preconditioned via the John ellipsoid. Our algorithm leverages the powerful response-based approachability framework of Bernstein and Shimkin (JMLR~'15)---previously overlooked in the line of work on swap regret minimization---and simultaneously minimizes \emph{profile swap regret}, which was recently shown to guarantee non-manipulability. Moreover, we establish a matching information-theoretic lower bound: any learner must incur in expectation $\Omega(d \sqrt{T})$ linear swap regret for large enough $T$, even when the set is centrally symmetric. This also shows that the classic algorithm of Gordon, Greenwald, and Marks (ICML~'08) is existentially optimal for minimizing linear swap regret, although it is computationally inefficient. Finally, we extend our approach to minimize regret with respect to the set of swap deviations with polynomial dimension, unifying and strengthening recent results in equilibrium computation and online learning.
\end{abstract}

\clearpage
\setcounter{page}{1}

\section{Introduction}

Online optimization provides a framework for sequential decision making under uncertainty. In this setting, a learner repeatedly selects a strategy from a convex set $\cP \subset \R^d$, while an adversary specifies a loss function. Although standard algorithms, such as online gradient descent~\citep{Zinkevich03:Online} and multiplicative weights update~\citep{Littlestone94:Weighted}, minimize \emph{external regret}---comparing the learner's performance relative to the best \emph{fixed} strategy in hindsight---this metric is often inadequate in complex dynamic environments. To address this, various strengthenings of external regret have been put forward. \emph{$\Phi$-regret} stands out as a central notion for capturing advanced forms of hindsight rationality~\citep{Stoltz07:Learning,Greenwald03:General}.

Unlike external regret, which juxtaposes the performance of the learner against a static comparator, $\Phi$-regret compares against a dynamic sequence of strategies produced by applying a transformation $\phi : \cP \to \cP$ from a set $\Phi$ to the learner's history of play. A learner experiences low $\Phi$-regret if no consistent modification of their past behavior---induced by some transformation $\phi \in \Phi$---would significantly reduce their cumulative loss. When $\Phi$ comprises all affine transformations mapping $\cP$ to itself, this is known as \emph{linear swap regret}. $\Phi$-regret has an intrinsic appeal as a general and robust performance measure in online learning, and has found many important applications in other areas.

Indeed, $\Phi$-regret has been a mainstay in algorithmic game theory, where it is deeply connected to fundamental solution concepts. A celebrated result establishes that if all players employ algorithms that minimize swap regret, the average correlated distribution of play converges to the set of \emph{correlated equilibria}~\citep{Foster97:Calibrated,Hart00:Simple,Stoltz05:Internal,Blum07:From}, a seminal notion introduced by~\citet{Aumann74:Subjectivity}. More broadly, no-$\Phi$-regret learning guarantees convergence to the set of \emph{$\Phi$-equilibria}~\citep{Stoltz07:Learning,Greenwald03:General}. While the early foundational work focused on the probability simplex, corresponding to \emph{normal-form} games, there has been a surge of recent research exploring $\Phi$-regret minimization in games with more complex strategy sets, most notably extensive-form games (see also \Cref{sec:related}).

Swap regret has also been linked to \emph{non-manipulability} against a strategic adversary~\citep{Deng19:Strategizing}. In fact, \citet{Arunachaleswaran24:Pareto} showed that, in stark contrast to mean-based algorithms such as multiplicative weights update, no-swap-regret algorithms are, in some sense, \emph{Pareto optimal}. The scope of this line of work has expanded to general polytope games~\citep{Mansour22:Strategizing}, culminating in the recent introduction of \emph{profile swap regret} by~\citet{Arunachaleswaran25:Swap}. Crucially, while minimizing full swap regret is intractable in the online learning setting~\citep{Daskalakis24:Lower}, \citet{Arunachaleswaran25:Swap} established that there are efficient algorithms for minimizing profile swap regret. Moreover, minimizing this metric suffices to guarantee that the learner cannot be manipulated by a strategic adversary.

From a computational point of view, developing efficient online algorithms to minimize $\Phi$-regret has been a formidable challenge. The first polynomial-time algorithm for minimizing linear swap regret over general convex sets was established only recently by~\citet{Daskalakis24:Efficient}. However, their algorithm yields a weak regret bound of $\Omega(d^4 \sqrt{T})$; this stands in stark contrast to the $O(\sqrt{T d \log d})$ bound achievable on the probability simplex~\citep{Blum07:From}. Furthermore, their approach relies on invoking the ellipsoid algorithm at each iteration, rendering it computationally intensive. Subsequent work extending the scope of~\citet{Daskalakis24:Efficient} beyond linear swap regret~\citep{Zhang25:Learning,Arunachaleswaran25:Swap} inherits these computational and information-theoretic limitations, leaving a significant gap in the literature.

\subsection{Our results} 
We introduce a new approach for minimizing linear swap regret and beyond over general convex sets. Our first main result is summarized below.

\begin{theorem}
    \label{theorem:ub-informal}
If $\cP \subset \R^d$ is a convex body, there is a computationally efficient algorithm that guarantees $O(d \sqrt{T})$ linear swap regret.
\end{theorem}

Prior to our work, the only computationally efficient algorithm for minimizing linear swap regret for general convex sets was the one by~\citet{Daskalakis24:Efficient}, which required $\Omega(d^{8}/\epsilon^2)$ rounds to guarantee that the learner incurs at most $\epsilon$ average linear swap regret; \Cref{theorem:ub-informal} provides a considerable improvement by a factor of $d^5$. Furthermore, our algorithm is significantly more efficient in terms of the per-iteration complexity. By virtue of the connection between linear swap regret and \emph{linear correlated equilibria}~\citep{Farina23:Polynomial}, \Cref{theorem:ub-informal} yields the fastest known algorithm for that problem in general convex games. Last but not least, our algorithm \emph{simultaneously} minimizes profile swap regret, thereby being non-manipulable~\citep{Arunachaleswaran25:Swap}.

Moreover, we establish a matching lower bound, showing that~\Cref{theorem:ub-informal} is existentially optimal for centrally symmetric convex bodies.

\begin{theorem}
\label{theorem:lb-informal}
For any $T = \Omega(d^4)$, there is a centrally symmetric convex body $\cP \subset \R^d$ and an adversary that guarantees that the expected linear swap regret of the learner is at least $\Omega(d \sqrt{T})$. 
\end{theorem}

This establishes a strict separation between minimizing linear swap regret over the simplex and a general convex set. It also shows that the classic algorithm of~\citet{Gordon08:No} is information-theoretically optimal for minimizing linear swap regret (\Cref{prop:Gordon}). One limitation of the algorithm behind~\Cref{theorem:ub-informal} is that it relies on a potentially inefficient preprocessing step, which we separate from the running time of the algorithm itself; an interesting question is whether the the $O(d \sqrt{T})$ rate attained by~\citet{Gordon08:No} can be matched without that limitation.

Finally, we extend our approach beyond linear (or profile) swap regret. We generalize~\Cref{theorem:ub-informal} to the class of swap deviations with \emph{polynomial dimension}~\citep{Zhang25:Learning,Gordon08:No}, which encompasses \emph{low-degree} swap deviations~\citep{Zhang24:Efficient}. As before, we significantly improve over those prior results in terms of the regret bounds. In fact, our algorithm simultaneously controls both the set of swap deviations with polynomial dimension and profile swap regret, unifying the two. The resulting equilibrium concept is the strongest currently known notion that, under suitable oracle access, can be efficiently computed in general convex games.

\paragraph{Technical approach} From a technical point of view, we crucially leverage the response-based algorithm of~\citet{Bernstein15:Response}. The relevance of their framework has been overlooked in recent developments on minimizing $\Phi$-regret. Our work brings it back to the forefront.

% Moreover, t
To obtain the $O(d \sqrt{T})$ upper bound on linear swap regret (\Cref{theorem:ub-informal}), we also perform a preprocessing step whereby the strategy set is brought into \emph{John's position} (\Cref{alg:precon-shimkin}). This allows us to tightly bound the norm of affine endomorphisms (\Cref{lemma:Frobeniusbound}) and the norm of points in the approachability space (\Cref{lemma:cK-bound}), paving the way to~\Cref{theorem:ub-informal}. Similarly, our extension to the set of swap deviations with polynomial dimension extends the basic framework by crucially leveraging \emph{mixed strategies} (\Cref{alg:poly-shimkin}).

Regarding our lower bound (\Cref{theorem:lb-informal}), we base our construction on the product strategy set $\cP = B_1 \times B_\infty$, where $B_1$ and $B_\infty$ denote the $\ell_1$ and $\ell_\infty$ balls, respectively. The core argument exploits a basic tradeoff: an adversary can compel the learner's strategy in the $B_1$ component to either incur large linear swap regret directly, or, in an effort to evade this, consistently cover a large fraction of the vertices (\Cref{lemma:movement}). In the latter case, we show that this enforced strategy diversity in $B_1$ strengthens the power of hindsight deviations available in the $B_\infty$ component (\Cref{lemma:punishment}): linear swap regret becomes a stronger benchmark as the learner's strategy sequence fluctuates.
\section{Preliminaries}
\label{sec:prels}

We denote by $\cP, \cL \subset \R^d$ the convex and compact strategy and loss sets, respectively. We make the usual normalization assumption $| \langle \ell, p \rangle | \leq 1$ for all $\ell \in \cL$ and $p \in \cP$. In particular, $\cL$ can be assumed to be the polar cone of $\conv(\cP, - \cP)$, namely $\{ \ell \in \cL : \max_{p \in \conv(\cP, -\cP)} \langle p, \ell \rangle \leq 1 \} $. For the set $\cP$, we can assume access to a linear optimization oracle~\citep{Grotschel12:Geometric}, which also yields a membership oracle for its polar cone $\cP^\circ$. We can also assume throughout that $0 \in \cP$.

\paragraph{Linear swap regret} We say that a function $\phi : \cP \to \R^d$ is an \emph{endomorphism} if $\phi(\cP) \subseteq \cP$. We denote by $\affend(\cP)$ the set of \emph{affine} endomorphisms with respect to $\cP$. A function $\phi \in \affend(\cP)$ can be represented as a matrix-vector pair $(M, a)$, so that $\phi(p) = M p + a$. For a sequence of strategies $( p_t )_{t=1}^T$ and losses $(\ell_t )_{t=1}^T$, the \emph{linear swap regret}~\citep{Gordon08:No} of the learner is defined as\footnote{\citet{Gordon08:No} did not include the affine component in their definition. When $0 \in \cP$, which is always the case when $\cP$ is centrally symmetric, selecting $p_t = 0$ guarantees zero linear swap regret. For this reason, we include the affine component.\label{fn:linearswap}}
\[
\LSR_T \defeq  \sum_{t=1}^T \langle \ell_t,p_t \rangle - \min_{\phi \in \affend(\cP)} \sum_{t=1}^T \langle \ell_t,\phi (p_t) \rangle.
\]

A general framework for minimizing $\Phi$-regret was developed by~\citet{Gordon08:No}. It distills the basic ingredients of earlier constructions---in the special case where $\cP$ is a probability simplex---by~\citet{Blum07:From} and~\citet{Stoltz05:Internal}. Specifically, the reduction of~\citet{Gordon08:No} posits access to i) a fixed-point oracle, which takes as input an endomorphism $\phi$ and returns a fixed point thereof; and ii) an external regret minimizer \emph{over the set $\Phi$}. For linear swap regret, each function $\phi$ is affine, so a fixed point can be computed efficiently. Beyond linear swap regret, one can do away with the fixed-point oracle required by~\citet{Gordon08:No} by using what~\citet{Zhang24:Efficient} refer to as \emph{expected} fixed points, which can be computed efficiently for any endomorphism. This is relevant to our extension to nonlinear swap deviations (\Cref{sec:swappoly}). On the other hand, minimizing external regret with respect to the set of affine endomorphisms turns out to be computationally intractable, as shown recently by~\citet{Daskalakis24:Efficient}. 

From an information-theoretic standpoint, we point out that the algorithm of~\citet{Gordon08:No} can yield a linear swap regret bound of $\tilde{O}(d \sqrt{T})$. Specifically, one can obtain $\tilde{O}(\sqrt{D T} )$ external regret over a general $D$-dimensional convex set using FTRL endowed with the \emph{entropic barrier} function~\citep{Abernethy08:Competing,Bubeck14:Entropic}, also known as continuous exponential weights~\citep[Section 3]{Bubeck11:Introduction}.\footnote{The loss vector produced by the reduction of~\citet{Gordon08:No} under affine endomorphisms can be cast as $L_t = (\ell_t \otimes p_t, \ell_t)$, so the condition $| \langle \phi_t, L_t \rangle | \leq 1$ is met.} Since the set of affine endomorphisms has dimension $d(d+1)$, we arrive at the following.

\begin{proposition}
    \label{prop:Gordon}
    Assuming that $|\langle p, \ell \rangle| \leq 1$ for all $p \in \cP$ and $\ell \in \cL$, there is an algorithm that guarantees $\LSR_T \leq \tilde{O}(d \sqrt{T})$ for any sequence of losses $\ell_1, \dots, \ell_T \in \cL$.
\end{proposition}

Our lower bound (\Cref{thm:lowerbound}) implies that this guarantee is information-theoretically optimal in the regime where $T$ is large enough. 

\paragraph{Profile swap regret} Profile swap regret was recently introduced by~\citet{Arunachaleswaran25:Swap}, who showed that the property of having no profile swap regret is equivalent to the property of being non-manipulable by a dynamic optimizer. Profile swap regret is defined with respect to the \emph{correlated strategy profile (CSP)} $\bk_T = \frac{1}{T} \sum_{t=1}^T \ell_t \otimes p_t$. If $R: \cP \times \cL \ni (p, \ell) \mapsto \max_{p^* \in \cP} \langle p - p^*, \ell \rangle $ measures the \emph{instantaneous regret}, we have
\[
    \PSR_T(\bk_T) = T \left( \min_{(p_{(t)}, \ell_{(t)}, \lambda_t)} \sum_{t=1}^T \lambda_t R(p_{(t)}, \ell_{(t)}) \right),
\]
where the minimum above is over all valid convex decompositions of $\bk_T$ into $\sum_{t=1}^T \lambda_t (\ell_{(t)} \otimes p_{(t)} )$ such that $\lambda_1, \dots, \lambda_T \geq 0$ with $\sum_{t=1}^T \lambda_t = 1$, $\ell_{(1)}, \dots, \ell_{(T)} \in \cL$, and $p_{(1)}, \dots, p_{(T)} \in \cP$.
\section{Efficient linear swap regret minimization via approachability}

In this section, we provide a computationally efficient algorithm whose linear swap regret is bounded as $O(d \sqrt{T})$ (\Cref{thm:main-ub}). %To put this into context, the only computationally efficient algorithm prior to our work was developed by~\citet{Daskalakis24:Efficient} who refined the framework of~\citet{Gordon08:No}. However, their algorithm is highly suboptimal, both in terms of the cumulated linear swap regret and the per-iteration running time, as it relies on the ellipsoid algorithm as a basic subroutine. 
Our approach circumvents the obstacle of minimizing external regret over $\affend(\cP)$ by making a connection to earlier work in the approachability literature, specifically the framework of~\citet{Bernstein15:Response}, together with preconditioning techniques from convex geometry. All omitted proofs are deferred to~\Cref{sec:proofs}.

\subsection{Reducing linear swap regret to approachability}

We begin by reducing linear swap regret to a suitable approachability problem (\Cref{lemma:AL-red}), so that we can leverage the response-based algorithm of~\citet{Bernstein15:Response}. We define the best-response function as $b(\ell) = \arg\min_{p \in \cP} \langle \ell,p \rangle$, with ties broken arbitrarily. Furthermore, we let $\cK \defeq \conv \{ (\ell \otimes p, \ell) : \ell \in \cL, p \in \cP \}$, and define the \emph{target set} $\cS \subseteq \cK$ as $\cS = \conv\{(\ell \otimes b(\ell), \ell): \ell \in \cL\}$. For a sequence of strategies $(p_t)_{t=1}^T$ and losses $(\ell_t)_{t=1}^T$, we denote by $\kappa_t = (\ell_t \otimes p_t, \ell_t)$ and $\bk_T = \frac1T \sum_{t=1}^T \kappa_t$ the average correlated strategy profile. With this notation at hand, linear swap regret can be equivalently expressed as
\begin{align*}
    \sum_{t=1}^T \langle \ell_t,p_t \rangle - \min_{\phi \in \affend(\cP)} \sum_{t=1}^T \langle \ell_t,\phi (p_t) \rangle = T \max_{\phi = (M, a) \in \affend(\cP)} \langle (I_d, 0) - (M, a), \bk_T \rangle,
\end{align*}
where the inner product is component-wise. By convention, we write $\| (M, a) \|^2_F = \|M \|_F^2 + \|a \|_2^2$. We define the \emph{approachability loss}, measured in the Euclidean norm, as \[\AL_T \defeq \min_{s \in \cS} \|\bk_T-s \|_{F}.\] The following lemma shows that minimizing linear swap regret reduces to minimizing the approachability loss.

\begin{restatable}{lemma}{ALred}
    \label{lemma:AL-red}
    $\LSR_T \leq 2T \AL_T \cdot \max_{\phi \in \affend(\cP)} \|\phi \|_F  $.
\end{restatable}

Profile swap regret also reduces to minimizing the same approachability loss. Specifically, following~\citet{Arunachaleswaran25:Swap}, the goal is to bound $\PSD_T(\bk_T)$, the \emph{profile swap distance}, which measures the distance of $\bk_T$ from a CSP $s$ such that $\PSR_T(s) \leq 0$. Since this is equivalent to $s \in \cS$, profile swap distance is equal to the approachability loss.

\subsection{Response-based approachability}

Having reduced minimizing swap regret to an approachability problem (\Cref{lemma:AL-red}), we now leverage the algorithm of~\citet{Bernstein15:Response}, which in our setting can be cast as~\Cref{alg:shimkin}.

\paragraph{Algorithm description and analysis} As with other approachability algorithms, the response-based algorithm of~\citet{Bernstein15:Response} fundamentally relies on the \emph{Pythagorean lemma}, recalled next. Its proof follows simply by induction. %To explain its core idea, we first recall the Pythagorean lemma, whose proof follows by induction.

\begin{lemma}
    \label{lemma:Pythagorean}
    Let $v_1,\dots, v_T \in \R^d$ be a sequence of points each with Euclidean norm bounded by $B$. Suppose further that $\inp{\sum_{\tau = 1}^{t-1} v_\tau,v_t} \leq 0$ for all $t \in [T]$. Then $\norm{\sum_{\tau=1}^T v_\tau}_2 \leq B \sqrt{T}$.
\end{lemma}

\begin{algorithm}[t]
\caption{Response-based approachability~\citep{Bernstein15:Response} }
\label{alg:shimkin}
\DontPrintSemicolon
\KwIn{Horizon $T$, convex sets $\cP, \cL \subset \R^d$, best-response map $b: \cL \to \cP$}
\KwOut{Sequence of strategies $p_1, \dots, p_T$}

Initialize $U_0 \gets 0 \in \R^{d \times (d+1)}$\;

\For{$t=1$ \KwTo $T$}{
    Compute a pair of minimax strategies $(p_t, \ell^*_t)$ of the (bilinear) zero-sum game
    \[
    \min_{p \in \cP} \max_{\ell \in \cL} \inp{U_{t-1}, (\ell \otimes p, \ell)}
    \]
    
    Set $s_t \gets (\ell_t^* \otimes b(\ell_t^*), \ell_t^*)$\;
    
    \textbf{Play} the strategy $p_t$ and \textbf{observe} the loss $\ell_t$\;
    
    Set $\kappa_t \gets (\ell_t \otimes p_t, \ell_t)$\;

    Update $U_t \gets U_{t-1} + (\kappa_t - s_t)$\;
}
\end{algorithm}

In light of this lemma, the key of the algorithm is to ensure that $\inp{\bk_{t-1}-\bs_{t-1},\kappa_t-s_t} \leq 0$. Indeed, by~\Cref{lemma:Pythagorean}, this guarantees a bound on the approachability loss: $\min_{s \in \cS} \norm{\bk_T-s}_2 \leq \norm{\bk_T-\bs_T}_2 \leq 2 B / \sqrt{T}$, for $B = \max_{p \in \cP, \ell \in \cL} \|( \ell \otimes p, \ell) \|_F = \max_{\ell \in \cL} \|\ell \|_2 \max_{p \in \cP} \sqrt{\|p \|^2_2 + 1}$.

To guarantee the invariance $\inp{\bk_{t-1}-\bs_{t-1},\kappa_t-s_t} \leq 0$, \Cref{alg:shimkin} proceeds as follows. At each time step $t \in [T]$, it maintains internally a point $s_t \in \cS$. We let $\bs_t = \frac{1}{t} \sum_{\tau=1}^t s_\tau$, $\bk_t = \frac{1}{t} \sum_{\tau=1}^t \kappa_\tau$, and $U_t \defeq \sum_{\tau=1}^t (\kappa_\tau - s_\tau)$. To select $p_t$ and $s_t$ in such a way that, for any adversarial selection of $\ell_t$, $\inp{\bk_{t-1}-\bs_{t-1},\kappa_t-s_t} \leq 0$, \Cref{alg:shimkin} first solves a (bilinear) minimax game: $p_t \in \arg \min_{p \in \cP} \max_{\ell \in \cL} \langle U_{t-1}, ( \ell \otimes p, \ell) \rangle$, $\ell_t^* \in \arg \max_{\ell \in \cL} \min_{p \in \cP} \langle U_{t-1}, ( \ell \otimes p, \ell) \rangle$. Then it selects $s_t = (\ell_t^* \otimes b(\ell_t^*), \ell^*_t)$. The correctness of the approach follows from the minimax theorem.

\begin{restatable}{lemma}{Shimkininvariant}
    \label{lemma:shimkin-invariant}
    The points $p_t$ and $s_t$ selected by \Cref{alg:shimkin} satisfy $\inp{U_{t-1}, \kappa_t - s_t} \leq 0$ for any adversarial loss $\ell_t \in \cL$, where $\kappa_t = (\ell_t \otimes p_t, \ell_t)$.
\end{restatable}

We summarize the main guarantee of~\Cref{alg:shimkin} concerning linear swap regret below. The proof follows by combining~\Cref{lemma:AL-red,lemma:Pythagorean,lemma:shimkin-invariant}.

\begin{theorem}
    \label{theorem:basicShimkin}
    For any sequence of losses $\ell_1, \dots, \ell_T \in \cL$, \Cref{alg:shimkin} guarantees
    \begin{equation}
        \label{eq:LSR-Shimkin}
        \LSR_T \leq 4 \sqrt{T} \left( \max_{p \in \cP, \ell \in \cL} \|\ell \|_2  \sqrt{\|p \|^2_2 + 1}  \right) \left( \max_{\phi \in \affend(\cP)} \|\phi \|_F \right).
    \end{equation}
\end{theorem}

Given that $\PSD_T = \AL_T$, we also obtain the following.

\begin{corollary}
    For any sequence of losses $\ell_1, \dots, \ell_T \in \cL$, \Cref{alg:shimkin} guarantees
    \[
     \PSD_T \leq 2 \sqrt{T} \left( \max_{p \in \cP, \ell \in \cL} \|\ell \|_2 \sqrt{\|p \|^2_2 + 1} \right).   
    \]
\end{corollary}

For a comparison between~\Cref{alg:shimkin} and Blackwell's approachability we refer to~\Cref{sec:comparison}. Our next goal is to show that~\eqref{eq:LSR-Shimkin} can be upper bounded by $O(d \sqrt{T})$ using a suitable preconditioner.

\subsection{Geometric preconditioning}

We now employ a suitable preconditioning step in conjunction with~\Cref{alg:shimkin} to obtain improved upper bounds on linear swap regret. We first point out a simple lemma, showing that if $\cP$ is mapped to $A \cP$ and $\cL$ to $A^{-\top} \cL$ for an invertible linear transformation $A$, the linear swap regret is preserved.

\begin{restatable}{lemma}{invariance}
    \label{lemma:invariance}
    Let $A \in \R^{d \times d}$ be an invertible matrix. Consider the transformed strategy set $\cP' = A \cP \defeq \{ A p : p \in \cP \}$ and loss set $\cL' = A^{-\top} \cL \defeq \{ A^{-\top} \ell : \ell \in \cL \}$. For any sequence of strategies $(p_t)_{t=1}^T$ in $\cP$ and sequence of losses $(\ell_t)_{t=1}^T$ in $\cL$, define the transformed sequences $p'_t = A p_t$ and $\ell'_t = A^{-\top} \ell_t$ for all $t \in [T]$. Then the linear swap regret $\LSR_T$ of $(p_t)_{t=1}^T$ with respect to $(\ell_t)_{t=1}^T$ is equal to the linear swap regret $\LSR'_T$ of $(p_t')_{t=1}^T$ with respect to $(\ell_t')_{t=1}^T$.
\end{restatable}

In particular, if $| \langle p, \ell \rangle | \leq 1$ for all $p \in \cP$ and $\ell \in \cL$, then $| \langle p', \ell' \rangle | = | \langle A p, A^{-\top} \ell \rangle | \leq 1$ for all $p' \in \cP$ and $\ell' \in \cL'$. This lemma allows us to adapt the geometry of $\cP$ and $\cL$ so as to minimize the right-hand side of~\eqref{eq:LSR-Shimkin} in~\Cref{theorem:basicShimkin}.

\paragraph{Frobenius norm of endomorphisms}

We begin by bounding the Frobenius norm of $\phi \in \affend(\cP)$. Our basic approach is to bring $\cP$ into \emph{John's position}. A convex body $\cP$ is said to be in John's position if the maximum volume ellipsoid contained in $\cP$ is the $d$-dimensional Euclidean unit ball centered at $0$, denoted by $B_2$. Without loss of generality, we will assume that $\cP$ is fully dimensional.

\begin{theorem}[John's Theorem]
    \label{theorem:John}
    If $\cP$ is in John's position, there exist contact points $\xi_1, \dots, \xi_m \in \partial \cP \cap \partial B_2$ and weights $c_1, \dots, c_m > 0$ such that
    \begin{equation}
        \label{eq:john-decomp}
        \sum_{i=1}^m c_i \xi_i \otimes \xi_i = I_d \quad \text{and} \quad \sum_{i=1}^m c_i \xi_i = 0.
    \end{equation}
    Furthermore, if $\cP$ is centrally symmetric, $B_2 \subseteq \cP \subseteq \sqrt{d} B_2$.
\end{theorem}

In this context, we will first show that if $\cP$ is centrally symmetric and is in John's position, the Frobenius norm of each $\phi \in \affend(\cP)$ is upper bounded by $O(\sqrt{d})$ (\Cref{lemma:Frobeniusbound}). First, we point out that using the second implication of John's theorem, namely $B_2 \subseteq \cP \subseteq \sqrt{d} B_2$, does \emph{not} suffice to obtain this bound; a lopsided ellipsoid provides a counterexample.

\begin{restatable}{lemma}{lopsided}
There exists a convex set \(\cP \subset\mathbb{R}^d\) satisfying $B_2 \subseteq \cP \subseteq \sqrt{d} B_2$ and a linear map $M$ such that
\(M(\cP)\subseteq \cP\) but \(\|M\|_F = \Omega(d)\).
\end{restatable}

Moreover, if we only knew that $B_2 \subseteq \cP \subseteq O(d) B_2 $, which is the case when $\cP$ is put into isotropic position~\citep{Lovasz06:Simulated}, we could only guarantee a spectral norm bound of the form $\|M \|_2 \leq O(d)$; in turn, this would yield $\| M \|_F \leq O(d^{3/2})$. To obtain an improved $O(\sqrt{d})$ upper bound, we need to make crucial use of John's decomposition---\eqref{eq:john-decomp} in~\Cref{theorem:John}---as formalized below.

\begin{restatable}{lemma}{Frobbound}
    \label{lemma:Frobeniusbound}
    Let $\cP \subset \R^d$ be a convex body in John's position and $\phi(p) = Mp + a$ be an affine endomorphism of $\cP$. Then we have $\|M \|_F \leq \sqrt{d + \|a \|_2^2 } $. In particular, if $\cP$ is centrally symmetric, $\|a\|_2 \leq \sqrt{d}$ and $\|M\|_F \leq \sqrt{2 d}$.
\end{restatable}

While this $O(\sqrt{d})$ bound hinges on $\cP$ being centrally symmetric, we observe that one can more generally put $\conv(\cP, -\cP)$ into John's position, which is centrally symmetric by definition. The following lemma justifies that transformation.

\begin{restatable}{lemma}{symmetrization}
    \label{lemma:sym}
    Let $\phi : p \mapsto M p + a$ be an affine endomorphism on $\cP$ with $0 \in \cP$. Then $\psi: p \mapsto \frac{1}{3} M p + \frac{1}{3}a$ is an endomorphism on $\conv (\cP, - \cP )$. 
\end{restatable}

As a result, if $\conv(\cP, -\cP)$ is in John's position, \Cref{lemma:Frobeniusbound} implies an $O(\sqrt{d})$ bound on the Frobenius norm of endomorphisms on $\cP$, as desired.

\paragraph{Frobenius norm of $\cK$}

Our next goal is to upper bound the Frobenius norm of points in $\cK$. If we know that $|\langle p, \ell \rangle| \leq 1$, the naive bound on $\| \ell \otimes p \|_F$ is linear in the dimension $d$. We significantly improve this by again making use of John's position and~\Cref{theorem:John}.

\begin{restatable}{lemma}{cKbound}
    \label{lemma:cK-bound}
    If $\cP$ is a convex body such that $\conv(\cP, -\cP)$ is in John's position and $\cL$ is such that $| \langle p, \ell \rangle | \leq 1$ for all $p \in \cP$ and $\ell \in \cL$, we have $\max_{p \in \cP, \ell \in \cL} \| ( \ell \otimes p, \ell) \|_F \leq \sqrt{d + 1}$.
\end{restatable}

\begin{algorithm}[t]
\caption{Preconditioned response-based approachability}
\label{alg:precon-shimkin}
\DontPrintSemicolon
\setcounter{AlgoLine}{0}
\KwIn{Horizon $T$, convex body $\cP \subset \R^d$ and loss set $\cL \subset \R^d$ such that $\langle p, \ell \rangle \leq 1$ for all $p \in \cP$ and $\ell \in \cL$, best-response map $b: \cL \to \cP$}
\KwOut{Sequence of strategies $p_1, \dots, p_T \in \cP$}

\tcp{1. Preconditioning Step}
Compute an invertible matrix $A$ such that $\conv( A\cP, - A \cP)$ is in John's position\;

Let $\cP' \gets A\cP$ and $\cL' \gets A^{-\top}\cL$\;

Define the transformed best-response map $b'(\ell') \gets A b(A^{\top}\ell')$\;

Initialize $U_0 \gets 0 \in \R^{d \times (d+1)}$\;

\For{$t=1$ \KwTo $T$}{
    \tcp{Solve the minimax game in the transformed space}
    Compute a pair of minimax strategies $(p'_t, \ell'^*_t)$ of the (bilinear) zero-sum game
    \[
    \min_{p' \in \cP'} \max_{\ell' \in \cL'} \inp{U_{t-1}, (\ell' \otimes p', \ell')}
    \]

    Set $s'_t \gets (\ell'^*_t \otimes b'(\ell'^*_t), \ell'^*_t)$\;
    
    \tcp{Play in the original space}
    \textbf{Play} the strategy $p_t \gets A^{-1} p'_t \in \cP$ and \textbf{observe} the loss $\ell_t \in \cL$\;
    
    \tcp{Update using transformed loss}
    Set transformed loss $\ell'_t \gets A^{-\top} \ell_t$\;
    
    Set $\kappa'_t \gets (\ell'_t \otimes p'_t, \ell'_t)$\;
    
    Update $U_t \gets U_{t-1} + (\kappa'_t - s'_t)$\;
}
\end{algorithm}

Combining~\Cref{lemma:invariance,lemma:Frobeniusbound,lemma:sym,lemma:cK-bound} with~\Cref{theorem:basicShimkin}, we obtain an online algorithm whose linear swap regret grows as $O(d \sqrt{T})$. This is given as~\Cref{alg:precon-shimkin}, which is a preconditioned version of~\Cref{alg:shimkin}.

\begin{theorem}
    \label{thm:main-ub}
    Let $\cP \subset \R^d$ be a convex body and a set of losses $\cL \subset \R^d$ such that $| \langle p, \ell \rangle | \leq 1$ for all $p \in \cP$ and $\ell \in \cL$. For any sequence of losses $\ell_1, \dots, \ell_T \in \cL$, \Cref{alg:precon-shimkin} guarantees 
    \[
      \LSR_T \leq O( d \sqrt{T}).
    \]
\end{theorem}

In terms of computational efficiency, the per-iteration complexity is dominated by the computation of minimax strategies, which can be solved efficiently using linear optimization oracles. Moreover, we show that~\Cref{thm:main-ub} extends when $(p_t, \ell_t^*)$ is an $\epsilon$-approximate equilibria with degradation proportional to $\epsilon T$ (\Cref{theorem:approxShimkin} in~\Cref{sec:approx-eq}); when $\epsilon \propto 1/\sqrt{T}$, this error term becomes negligible. The problem of approximate minimax equilibrium computation can be solved in a scalable fashion via (external) no-regret learning~\citep{Freund97:decision}. 

One limitation of~\Cref{thm:main-ub} is that the John ellipsoid cannot be computed efficiently with just oracle access to the set. (On the other hand, it is well-known that the John ellipsoid can be computed in polynomial time for explicitly represented polytopes using convex optimization~\citep{Kumar05:Minimum,Todd07:Khachiyan,Cao22:Faster}.) This preprocessing step is only applied once, and we separate that from the running time of the algorithm itself. It is an interesting question whether the preprocessing step can also be made efficient while maintaining the $O(d \sqrt{T})$ bound on linear swap regret.

%Without assuming that $\cP$ is centrally symmetric, we take $\Phi$ to contain linear endomorphisms, without the affine component (\Cref{fn:linearswap}). The bound in~\Cref{lemma:cK-bound} now grows linearly in $d$ (\Cref{lemma:weak-bound}), so the linear swap regret we obtain is inferior by a $\sqrt{d}$ factor.

%\begin{corollary}
%    \label{cor:nosym}
%    Let $\cP \subset \R^d$ be a convex body and a set of losses $\cL \subset \R^d$ such that $\langle p, \ell \rangle \leq 1$ for all $p \in \cP$ and $\ell \in \cL$. For any sequence of losses $\ell_1, \dots, \ell_T \in \cL$, there is a computationally efficient algorithm whose linear swap regret grows as $O(d^{3/2} \sqrt{T})$.
%\end{corollary}

%This leaves a $\sqrt{d}$ gap compared to~\Cref{thm:main-ub} for centrally symmetric bodies and the algorithm of~\citet{Gordon08:No} (\Cref{prop:Gordon}), although the latter is computationally inefficient.

Prior to our work, the only computationally efficient algorithm for minimizing linear swap regret over general convex sets was the one developed by~\citet{Daskalakis24:Efficient}, which gave a bound of $O(d^4 \sqrt{T})$. In terms of the number of iterations needed to guarantee at most $\epsilon$ linear swap regret, their algorithm requires $T = \Omega(d^8/\epsilon^2)$, whereas ours requires at most $16 d^3/\epsilon^2$. This is an improvement by a factor of $d^5$. Furthermore, \Cref{alg:precon-shimkin} is significantly more efficient per iteration.
\section{Swap deviations with polynomial dimension}
\label{sec:swappoly}

In this section, we extend our approach beyond linear and profile swap regret. Specifically, we consider the class of swap deviations with \emph{polynomial dimension}, introduced in the form below by~\citet{Zhang25:Learning}, although it goes back to~\citet{Gordon08:No}.

\begin{definition}[Swap deviations with polynomial dimension; \citealp{Zhang25:Learning}]
    \label{def:polydim}
    Let $m : \cP \to \R^D$ be a feature map. $\Phi^m$ is the set of all endomorphisms $\phi$ that can be expressed by the matrix-vector product $M m(p)$ for some $M = M(\phi) \in \R^{d \times D} $. The set $\Phi^m$ has dimension at most $k \defeq D \cdot d$.
\end{definition}
We do not include the affine component since one can simply append an additional coordinate of $1$ to $m(p)$ to capture that. A canonical example of~\Cref{def:polydim} is \emph{$\ell$-degree polynomials}, which can be obtained when $m(p)$ outputs all $\ell$-wise (and lower) products of $p$ and $M$ contains the coefficients of the polynomials. In what follows, we assume that the first $d$ coordinates of $m(p)$ contain $p$ itself.

Our goal is to minimize $\Phi^m$-regret with respect to a set $\Phi^m$ with dimension $k \leq \poly(d)$. Even when minimizing regret with respect to low-degree swap deviations, \citet{Zhang24:Efficient} leveraged an earlier reduction by~\citet{Hazan07:Computational} to show that the problem is intractable---specifically, $\mathsf{PPAD}$-hard---unless the learner can commit to \emph{mixed} strategies, that is, distributions in $\cP$. As a result, it will be essential to consider mixed strategies. For a sequence of losses $\ell_1, \dots, \ell_T \in \cL$, the $\Phi^m$-regret of a learner who plays a sequence of mixed strategies $\mu_1, \dots, \mu_T \in \Delta(\cP)$ is defined as
\begin{equation}
    \label{eq:PDSR}
    \PDSR_T = \sum_{t=1}^T \E_{p_t \sim \mu_t} \langle \ell_t, p_t \rangle - \min_{\phi \in \Phi^m} \sum_{t=1}^T \E_{p_t \sim \mu_t} \langle \ell_t, \phi (p_t)\rangle.
\end{equation}

\subsection{Reducing to approachability}

To reduce this problem into an approachability instance, we define the best-response map $B(\ell) = m(b(\ell))$, where, as before, $b(\ell) = \arg\min_{p \in \cP} \inp{\ell,p}$, and the target set is defined as $\cS = \conv\{\ell \otimes B(\ell) : \ell \in \cL \}$. Furthermore, for a sequence of distributions $\mu_{1}, \dots, \mu_T \in \Delta(\cP)$ and losses $\ell_{1}, \dots, \ell_T \in \cL$, we let $\kappa_t = \E_{p_t \sim \mu_t} \ell_t \otimes m(p_t)$ and $\bk_T = \frac1T \sum_{t=1}^T \kappa_t$. Let $J_d \in \R^{D \times d}$ be a matrix that contains $1$ in the diagonal and $0$ everywhere else. We can equivalently write~\eqref{eq:PDSR} as
\begin{align*}
    \PDSR_T &= \sum_{t=1}^T \E_{p_t \sim \mu_t} \inp{\ell_t, J_d m(p_t)} - \min_{\phi \in \Phi^m} \sum_{t=1}^T \E_{p_t \sim \mu_t} \inp{\ell_t,\phi m(p_t)} \\
    &= T \max_{M \in \Phi^m} \inp{J_d- M ,\bk_T},
\end{align*}
where the above derivation used the assumption that the first $d$ coordinates of $m(p)$ contain $p$ itself. As before, we define the approachability loss measured in the Euclidean norm as $\AL_T = \min_{s \in \cS}  \| \bk_T-s \|_{F}$. The following lemma follows analogously to~\Cref{lemma:AL-red}.

\begin{restatable}{lemma}{polyALred}
    \label{lemma:poly-AL-red}
    $\PDSR_T \leq 2T \AL_T \cdot \max_{M \in \Phi^m} \| M \|_F   $.
\end{restatable}

In particular, the approachability loss in this instance upper bounds both $\PDSR_T$ and $\PSR_T$, unifying~\citet{Zhang25:Learning} and~\citet{Arunachaleswaran25:Swap}. This gives rise to the strongest notion of correlated equilibrium that can be efficiently computed in games.

To solve the induced approachability problem, we again rely on~\Cref{alg:shimkin}. As before, the key is to maintain the invariance $\inp{\bk_{t-1}-\bs_{t-1},\kappa_t-s_t} \leq 0$. To do so, it suffices to show that for any matrix $U \in \R^{d \times D}$, there exists $s \in \cS$ and $\mu \in \Delta(\cP)$ such that, no matter the choice of $\ell \in \cL$, $\E_{p \sim \mu} \langle U, \ell \otimes m(p) - s \rangle \leq 0$. We can obtain this guarantee by computing a pair of minimax equilibrium strategies $(\mu_t, \ell_t^*)$ of the zero-sum game
\[
\min_{\mu \in \Delta(\cP)} \max_{\ell \in \cL} \E_{p \sim \mu} \ell^\top U_{t-1} m(p),
\]
where $U_{t-1} = \sum_{\tau=1}^{t-1} (\kappa_\tau - s_\tau)$. The key point is that this is bilinear in terms of $\mu$ and $\ell$, and the minimax theorem can be applied. This is where working with mixed strategies is crucial; otherwise, we would be dealing with the problem $\min_{p \in \cP} \max_{\ell \in \cL} \ell^\top U m(p)$, for which we cannot use the minimax theorem. This is not a limitation of the approachability framework: \citet{Zhang24:Efficient} showed that producing a sequence $p_1, \dots, p_T$ that minimizes regret even with respect to degree-$2$ swap deviations is intractable. Our construction mirrors~\Cref{alg:shimkin}, and is given in~\Cref{alg:poly-shimkin}.

\begin{algorithm}[t]
\caption{Response-based approachability for nonlinear swap deviations }
\label{alg:poly-shimkin}
\DontPrintSemicolon
\setcounter{AlgoLine}{0}
\KwIn{Horizon $T$, convex and compact sets $\cP, \cL \subset \R^d$, feature map $m: \cP \to \R^D$, best-response map $b: \cL \to \cP$}
\KwOut{Sequence of mixed strategies $\mu_1, \dots, \mu_T \in \Delta(\cP)$}

Initialize $U_0 \gets 0 \in \R^{d \times D}$\;

\For{$t=1$ \KwTo $T$}{
    \tcp{Compute a minimax equilibrium in mixed strategies}
    Compute a pair of minimax strategies $(\mu_t, \ell^*_t)$ of the (bilinear) zero-sum game
    \[
    \min_{\mu \in \Delta(\cP)} \max_{\ell \in \cL} \E_{p \sim \mu} \ell^\top U_{t-1} m(p) 
    \]
    
    Set $s_t \gets \ell_t^* \otimes m(b(\ell_t^*))$\;
    
    \textbf{Play} the mixed strategy $\mu_t$ and \textbf{observe} the loss $\ell_t$\;
    
    Set $\kappa_t \gets \E_{p_t \sim \mu_t} [\ell_t \otimes m(p_t)]$\;
    
    Update $U_t \gets U_{t-1} + (\kappa_t - s_t)$\;
}
\end{algorithm}

\begin{restatable}{lemma}{shimkinmixed}
    \label{lemma:shimkin-invariant-mixed}
    The distribution $\mu_t \in \Delta(\cP)$ and the point $s_t$ selected by~\Cref{alg:poly-shimkin} satisfy $\langle U_{t-1}, \kappa_t - s_t \rangle \leq 0$ for any adversarial loss $\ell_t \in \cL$.
\end{restatable}

Combining~\Cref{lemma:poly-AL-red,lemma:Pythagorean,lemma:shimkin-invariant-mixed}, we arrive at the following result.

\begin{theorem}
    \label{theorem:polyShimkin}
    For any sequence of losses $\ell_1, \dots, \ell_t \in \cL$, \Cref{alg:poly-shimkin} guarantees 
    \[
    \PDSR_T \leq 4 \sqrt{T} \left( \max_{p \in \cP} \|m(p) \|_2 \max_{\ell \in \cL} \|\ell \|_2  \right) \left( \max_{M \in \Phi^m} \|M \|_F \right).
    \]
\end{theorem}

This significantly improves over the algorithm of~\citet{Zhang25:Learning} in terms of the regret bound. On the other hand, one question we leave open for future work is under what conditions we can efficiently solve the bilinear saddle-point problem in~\Cref{alg:poly-shimkin} with only oracle access to $\cP$; the challenge is that optimizing over the set of distributions $\Delta(\cP)$ could be more complex.
\section{A matching lower bound for linear swap regret}
\label{sec:lowerbound}

Finally, we construct an adaptive adversary that forces any learner to incur $\Omega(d \sqrt{T})$ linear swap regret in expectation. We show this even when $\cP$ is centrally symmetric, matching~\Cref{thm:main-ub}. Our lower bound also implies that the algorithm of~\citet{Gordon08:No} is existential optimal algorithm for minimizing linear swap regret (with or without central symmetry) in view of~\Cref{prop:Gordon}.

To put this into context, we recall that for the simplex $\cP = \Delta^d$, the most well-studied strategy set in the context of minimizing swap regret, there are algorithms that guarantee $O(\sqrt{T d \log d})$ (linear) swap regret~\citep{Blum07:From}, so our lower bound establishes a strict separation.
\begin{theorem}
    \label{thm:lowerbound}
    Let $T = \Omega(d^4)$ and $\cP \subset \R^d$ a centrally symmetric convex body. There exists an adaptive adversary that guarantees $\E[\LSR_T] \geq \Omega(d \sqrt{T})$.
\end{theorem}

We now sketch the main steps in the proof of \cref{thm:lowerbound}; details are available in \cref{sec:proofs}.

\paragraph{The setup} The strategy set we consider is $\cP = B_1 \times B_\infty$; here, $B_1$ is the $\ell_1$ ball with radius 1 centered at $0$ and $B_\infty$ is the $\ell_\infty$ ball with radius 1 centered at $0$. We denote the strategy of the learner in each round $t \in [T]$ by $(x_t, p_t)$, where $x_t \in B_1$ and $p_t \in B_\infty$. Furthermore, the loss is written as $(y_t, \ell_t)$, where $y_t \in B_\infty$ and $\ell_t \in B_1$. In other words, $x_t$ is pitted against $y_t$ and $p_t$ against $\ell_t$. This guarantees that $\langle x_t, y_t \rangle \leq 1$ and $\langle p_t, \ell_t \rangle \leq 1$.\footnote{Technically, this only guarantees $\langle x_t, y_t \rangle +  \langle p_t, \ell_t \rangle \leq 2$, but one can simply rescale the losses by a factor of $2$.} Furthermore, $\cP$ is centrally symmetric.

Our construction relies on a subset of $\affend(\cP)$, namely transformations of the form
\begin{equation}
    \label{eq:restr-end}
    \phi(x,p) = 
    \begin{bmatrix}
    A & 0\\
    M & 0
\end{bmatrix}
\begin{bmatrix}
x \\ 
p
\end{bmatrix}
+
\begin{bmatrix}
    x^*\\
    0
\end{bmatrix},
\end{equation}
where $M \in [-1, 1]^{d \times d}$ and $A$ and $x^*$ are such that $A x + x^* \in B_1$ for all $x \in B_1$. For any $M \in [-1, 1]^{d \times d}$ and $x \in B_1$, it follows that $M x \in B_\infty$, so any function per~\eqref{eq:restr-end} is indeed a valid endomorphism. Let $\Phi \subset \affend(\cP)$ be the set of all such functions. The $\Phi$-regret of the learner can be expressed as 
\[
\phireg_T = \sum_{t=1}^T \langle y_t, x_t - ( A x_t + x^*) \rangle + \sum_{t=1}^T \langle \ell_t, p_t - M x_t \rangle.
\]
The idea in our construction is this: the adversary we will use the sequence of losses $y_1, \dots, y_T \in B_\infty$ to force the sequence $x_1, \dots, x_T$ to have two key properties: i) a large fraction of the strategies are close to a vertex of $\Delta^d$ (local concentration), and ii) a large fraction of the vertices are played often with considerable probability (global anti-concentration). Any learner whose linear swap regret is sublinear will be forced to exhibit these properties (\Cref{lemma:movement}). Furthermore, when these conditions are satisfied, we will show that the second term in $\Phi$-regret, $\sum_{t=1}^T \langle \ell_t, p_t - M x_t \rangle$, will be $\Omega(d \sqrt{T} )$ in expectation (\Cref{lemma:punishment}). In other words, the more the first part of the strategy $x_t$ mixes between vertices of $\Delta^d$, the harder it becomes to minimize linear swap regret in terms of $p_t$: (linear) swap regret becomes a stronger metric the more the learner changes their strategy.

\paragraph{Enforcing vertex coverage} We first construct an adversary with the following property.

\begin{restatable}{lemma}{movement}
    \label{lemma:movement}
    Let $T \geq 16 d^2 (d+1)^2$. There exists an adaptive adversary who produces a sequence of losses $y_1, \dots, y_T \in B_\infty$ such that one of the following holds.
    \begin{enumerate}
        \item $\sum_{t=1}^T \langle y_t, x_t - (A x_t + x^*) \rangle \geq d \sqrt{T}$ for some endomorphism $(A, x^*)$ on $B_1$.\label{item:largeswap}
        \item If $I \defeq \{ i \in [d] : \sum_{t = 1}^T \mathbbm{1} \left\{ x_{t, i} \geq \frac{1}{8} \right\} \geq \frac{T}{16 d} \}$, then $|I| \geq d/15$.\label{item:coverage}
    \end{enumerate}
\end{restatable}

The other sequence of losses we construct $\ell_1, \dots, \ell_T$ guarantees $\E \langle p_t, \ell_t \rangle = 0$, so we have 
\[
\E[\phireg_T] \geq \sum_{t=1}^T \langle y_t, x_t - ( A x_t + x^*) \rangle].
\] 
This means that~\Cref{item:largeswap} implies $\E[\phireg_T] \geq d \sqrt{T}$. The adversary behind~\Cref{lemma:movement} makes each of the $d$ vertices of $\Delta^d$ highly profitable in each of $d$ periods of equal length, while making all previously profitable vertices highly costly.

\paragraph{Punishing vertex coverage} To complete the construction, we show that if the sequence $x_1, \dots, x_T$ satisfies~\Cref{item:coverage}, the adversary can guarantee $\E[ \sum_{t=1}^T \langle \ell_t, p_t - M x_t \rangle ] \geq \Omega(d \sqrt{T})$.

\begin{restatable}{lemma}{punish}
    \label{lemma:punishment}
    Let $T \geq 16 d^2 (d+1)^2$. If the set $I \defeq \{ i \in [d] : \sum_{t = 1}^T \mathbbm{1} \left\{ x_{t, i} \geq \frac{1}{8} \right\} \geq \frac{T}{16 d} \}$ satisfies $|I| = \Omega(d)$, there is an oblivious adverasry who can produce a sequence of losses $\ell_1, \dots, \ell_T \in B_1$ such that 
    \[
    \E\left[ \max_{M \in [-1, 1]^{d \times d}} \sum_{t=1}^T \langle \ell_t, p_t - M x_t \rangle \right] \geq \Omega(d \sqrt{T}).
    \]
\end{restatable}

Combining~\Cref{lemma:movement} and~\Cref{lemma:punishment}, \Cref{thm:lowerbound} follows.
\section{Conclusions and future research}

We introduced a new unifying approach for minimizing different notions of swap regret by connecting to the response-based approachability framework of~\citet{Bernstein15:Response}. Compared to previous work of~\citet{Daskalakis24:Efficient}, the resulting algorithm attains an improved regret bound and is also computationally more efficient. Moreover, we established a matching $\Omega(d \sqrt{T})$ lower bound for linear swap regret, establishing for the first time that the classic algorithm of~\citet{Gordon08:No} is existentially optimal for that metric, although it is computationally intractable~\citep{Daskalakis24:Efficient}. An interesting question is to obtain stronger lower bounds for a set of swap deviations with polynomial dimension, such as low-degree polynomials. Moreover, our focus has been on the regime where $T$ is significantly larger than the dimension. It is possible that one can improve over the $\Omega(d \sqrt{T})$ barrier in the high dimensional regime by leveraging recent advances~\citep{Dagan24:From,Peng24:Fast}.

\section*{Acknowledgments}
We are grateful to Brian Hu Zhang for insightful discussions. HL is supported by NSF award IIS-1943607.
GF was supported in part by the National Science Foundation award CCF-2443068, the Office of Naval Research grant N000142512296,
and the AI2050 Early Career Fellowship.

\bibliography{refs, ref}

\clearpage

\appendix

\section{Further related work}
\label{sec:related}

The notion of linear swap regret---as a relaxation of full swap regret---was introduced by~\citet{Gordon08:No}. The first efficient algorithm for minimizing linear swap regret in extensive-form games was provided more recently by~\citet{Farina23:Polynomial}. In that setting, linear swap regret subsumes weaker notions of regret \citep{Farina22:Simple} linked to extensive-form correlated and coarse correlated equilibria \citep{Stengel08:Extensive,Farina20:Coarse}. A natural interpretation of linear correlated equilibria based on a mediator was later provided by~\citet{Zhang23:Mediator} (\emph{cf.}~\citealp{Fujii25:Bayes}). \citet{Daskalakis24:Efficient} observed that when the underlying strategy set $\cP$ is given explicitly as a polytope, there is an efficient membership oracle that ascertains whether a purported affine endomorphism is indeed one. By standard reductions~\citep{Grotschel12:Geometric}, this enables efficient optimization over $\affend(\cP)$, and thereby efficient regret minimization. When $\cP$ is given explicitly as a polytope, \citet{Zhang25:Expected} provided an explicit representation of the set $\affend(\cP)$, which enables efficient algorithms without relying on ellipsoid.

Algorithms for minimizing $\Phi$-regret have traditionally either followed the template of~\citet{Gordon08:No}, or rely on the approachability framework of~\citet{Blackwell56:analog}. While these approaches are in a certain sense equivalent~\citep{Abernethy11:Blackwell}, there is nuance to this equivalence~\citep{Dann2025:Rate}. A new algorithmic approach for minimizing swap regret was recently introduced by~\citet{Peng24:Fast} and~\citet{Dagan24:From}, but their reduction is only meaningful in the very high dimensional regime. Interestingly, our $\Omega(d \sqrt{T})$ lower bound for linear swap regret only applies when $T \geq \Omega(d^4)$; it is possible that in the high dimensional regime one can improve over our $O(d \sqrt{T})$ upper bound using some variation of the approach of~\citet{Peng24:Fast,Dagan24:From} tailored to linear swap regret.

Tracing back to the work of~\citet{Foster97:Calibrated}, it has been known that swap regret is intimately connected to calibration. This connection was further cultivated in recent work that addressed the complexity of calibration in high dimensions~\citep{Fishelson25:High,Peng25:High}. Relatedly, the notion of swap regret has found fertile ground in the line of work on multicalibration~\citep{Hebert18:Multicalibration} and omniprediction~\citep{Gopalan21:Omnipredictors,Gopalan23:Swap}.

The famous theorem of~\citet{John48:Extremum} states that every compact and convex set with nonempty interior has a unique maximum-volume inscribed ellipsoid---known as the John ellipsoid. This concept has been a mainstay in convex geometry~\citep{Ball91:Volume,Ball01:Convex,Lutwak05:John,Todd16:Minimum}, and has many important algorithmic applications (\emph{e.g.},~\citealp{Lee14:Path,Hazan16:Volumetric,Nikolov13:Geometry,Tang24:Uncertainty}).
\section{Comparison of different approachability algorithms}
\label{sec:comparison}

In this section, we contrast~\Cref{alg:shimkin} with Blackwell's algorithm. While both rely at their core on the Pythagorean lemma (\Cref{lemma:Pythagorean}), the key distinction is that Blackwell's algorithm separates $s$ from $(\ell \otimes p, \ell)$ via the supporting hyperplane whereas~\Cref{alg:shimkin} separates via a hyperplane produced by solving the auxiliary minimax game (\Cref{fig:algs}).

\Cref{fig:algs} depicts both $\cS$ and $\cK$. $U$ is pointing to the right, so the goal is guarantee that $(\ell \otimes p, \ell)$ is left of $s$. We show two hyperplanes orthogonal to $U$: one is tangent to $\cS$ (support) and the other (minimax) corresponds to $\inp{U,\kappa} = v$ where $v$ is the value of the minimax game $\min_{p \in \cP} \max_{\ell \in \cL} \langle U, (\ell \otimes p, \ell) \rangle $. The minimax hyperplane is necessarily left of the supporting hyperplane.

\begin{figure}
    \centering
    \scalebox{0.8}{\begin{tikzpicture}[scale=1]
  % outer convex set
  \draw[thick,fill=gray!10]
    (0,0) ellipse (3cm and 2cm)
    node[right,xshift=0.3cm,yshift=0.3cm] {$\mathcal{K}$};

  % inner convex set S
  \draw[thick,fill=white]
    (-1.5,0) ellipse (1.5cm and 1cm)
    node[below] {$\cS$};

  % red tangent line
  \draw[red,thick]
    (0cm,-2.9cm) -- (0cm,2.9cm)
    node[above,black] {support};

  % blue dashed line through center
  \draw[blue,thick,dashed]
    (-1cm,-2.5cm) -- (-1cm,2.5cm)
    node[above,black] {minimax};

  % arrow labeled U
  \draw[->,thick]
    (-2,-3.0) -- (2,-3.0)
    node[midway,below] {$U$};
\end{tikzpicture}}
    \caption{Illustration of how~\Cref{alg:shimkin} produces a separating hyperplane (minimax) versus Blackwell's choice (support).}
    \label{fig:algs}
\end{figure}
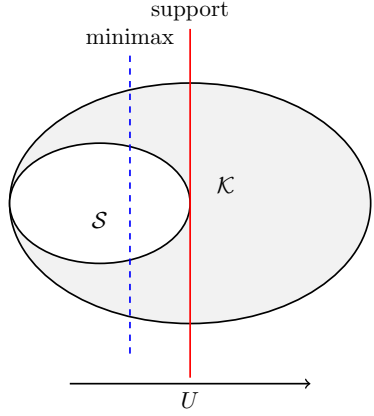

In Blackwell's algorithm, $U_t = \bk_{t} - \text{Proj}_{S}(\bk_{t})$, and $s_t \in \arg \max_{s' \in \cS} \inp{U_{t-1},s'} $ is selected as the point of tangency of the supporting hyperplane. Then $p_t \in \cP$ is chosen such that for any adversarial loss $\ell_t \in \cL$, $(\ell_t \otimes p_t, \ell_t)$ is left of the supporting hyperplane; the existence of such a point $p_t$ is guaranteed by the preconditions of approachability.

The main advantage of~\Cref{alg:shimkin} is that it only relies on a best-response oracle $b : \cL \to \cP$. Compared to Blackwell's algorithm, the choice of $s_t$ is in a certain sense weaker, but the selection of $p_t$ still ensures that $\kappa_t$ will be on the left of the minimax hyperplane. Crucially, this only involves optimizing over $\cP$.
\section{Extension under approximate minimax equilibria}
\label{sec:approx-eq}

We are concerned in this section with the approximate version of~\Cref{alg:shimkin} whereby $(p_t, \ell_t^*)$ is an $\epsilon_t$-approximate minimax equilibrium (\Cref{alg:shimkin-approx}). We say that $(p_t, \ell_t^*)$ is an $\epsilon_t$-approximate minimax equilibrium if the duality gap---the sum of the players' best-response gaps---is at most $\epsilon_t$. Compared to~\Cref{alg:shimkin}, \Cref{alg:shimkin-approx} normalizes the payoff matrix.

To account for this error in the analysis, we begin by extending the Pythagorean lemma (\Cref{lemma:Pythagorean}). 

\begin{lemma}
    \label{lemma:approx-Pythagorean-normalized}
    Let $v_1,\cdots, v_T \in \R^d$ be a sequence of points each with Euclidean norm bounded by $B$. Suppose further that for all $t \geq 2$,
    \[
        \left\langle \frac{1}{t-1} \sum_{\tau=1}^{t-1} v_{\tau} , v_t \right\rangle \leq \epsilon_t
    \]
    for some $\epsilon_t \geq 0$. Then $\left\|\sum_{\tau=1}^T v_\tau \right\|_2 \leq \sqrt{T B^2 + 2 \sum_{t=2}^T (t-1) \epsilon_t}$.
\end{lemma}

\begin{proof}
    Let $S_t \defeq \sum_{\tau=1}^t v_\tau$. We have
    \begin{align*}
        \|S_t\|_2^2 &= \|S_{t-1} + v_t\|_2^2 \\
        &= \|S_{t-1}\|_2^2 + \|v_t\|_2^2 + 2 \inp{S_{t-1}, v_t}.
    \end{align*}
    For $t=1$, $\|S_1\|_2^2 = \|v_1\|_2^2 \leq B^2$. For $t \geq 2$, we use the assumption that $\langle S_{t-1}, v_t \rangle \leq (t-1) \epsilon_t$ to bound
    \[
        \|S_t\|_2^2 \leq \|S_{t-1}\|_2^2 + B^2 + 2(t-1) \epsilon_t.
    \]
    Summing this inequality over $t=1, \dots, T$, the sum telescopes, yielding
    \[
        \|S_T\|_2^2 \leq \sum_{t=1}^T B^2 + 2 \sum_{t=2}^T (t-1) \epsilon_t = T B^2 + 2 \sum_{t=2}^T (t-1) \epsilon_t.
    \]
    Taking the square root completes the proof.
\end{proof}

We next adapt~\Cref{lemma:shimkin-invariant}.

\begin{lemma}
    The points $p_t$ and $s_t$ selected by \Cref{alg:shimkin-approx} satisfy $\inp{\bar{U}_{t-1}, \kappa_t - s_t} \leq \epsilon_t$ for any adversarial loss $\ell_t \in \cL$, where $\kappa_t = (\ell_t \otimes p_t, \ell_t)$.
\end{lemma}

\begin{proof}
    Let $g_t(p, \ell) \defeq \inp{\bar{U}_{t-1}, (\ell \otimes p, \ell)}$ be the payoff function of the auxiliary minimax game at step $t$, as defined in~\Cref{alg:shimkin-approx}. Since $\cP$ and $\cL$ are both convex and compact and $g$ is bilinear, the minimax theorem applies~\citep{vonNeumann28:Zur}. We let $V_t$ be the value of the game,
    \[
        V_t \defeq \min_{p \in \cP} \max_{\ell \in \cL} g_t(p, \ell) = \max_{\ell \in \cL} \min_{p \in \cP} g_t(p, \ell).
    \]
    We let $\epsilon_{p_t}$ and $\epsilon_{\ell_t}$ be the best-response gap of the two players, respectively, which satisfy $\epsilon_{p_t} + \epsilon_{\ell_t} = \epsilon_t$. It follows that for any (realized) loss $\ell_t \in \cL$,
    \begin{equation*}
        \inp{\bar{U}_{t-1}, \kappa_t} = g_t(p_t, \ell_t) \leq V_t + \epsilon_{p_t}.
    \end{equation*}
    Furthermore, 
    \begin{equation*}
        V_t \leq \min_{p \in \cP} \inp{\bar{U}_{t-1}, (\ell_t^* \otimes p, \ell_t^*)} + \epsilon_{\ell_t} \leq \inp{\bar{U}_{t-1}, (\ell_t^* \otimes b(\ell_t^*), \ell_t^*)} + \epsilon_{\ell_t} = \inp{\bar{U}_{t-1}, s_t} + \epsilon_{\ell_t}.
    \end{equation*}
    Combining these inequalities completes the proof.
\end{proof}

Combining with~\Cref{lemma:approx-Pythagorean-normalized}, it follows that 
\[
\min_{s \in \cS} \norm{\bk_T-s}_2 \leq \norm{\bk_T-\bs_T}_2 \leq \sqrt{ \frac{4 B^2}{T} + \epsilon} \leq \frac{2B}{\sqrt{T}} + \epsilon,    
\]
where $B =  \max_{\ell \in \cL} \|\ell\|_2 \max_{p \in \cP} \sqrt{\|p\|_2^2 + 1}$ and we assumed $\epsilon_t = \epsilon$ for all $t \in [T]$. We thus arrive at the following refinement of~\Cref{theorem:basicShimkin}.

\begin{theorem}
    \label{theorem:approxShimkin}
    For any sequence of losses $\ell_1, \dots, \ell_t \in \cL$, \Cref{alg:shimkin-approx} with $\epsilon_t = \epsilon$ for all $t \in [T]$ guarantees
    \begin{equation*}
        \LSR_T \leq 2 \left( 2 \sqrt{T} \left( \max_{p \in \cP, \ell \in \cL} \|\ell \|_2 \sqrt{ \|p \|_2^2 + 1} \right) + \epsilon T \right) \left( \max_{\phi \in \affend(\cP)} \|\phi \|_F \right).
    \end{equation*}
    In particular, if $\cP$ is a centrally symmetric convex body in John's position,
    \[
        \LSR_T \leq 2 \sqrt{2 d} \left( 2 \sqrt{T} \left( \sqrt{d + 1} \right) + \epsilon T \right).
    \]
\end{theorem}

As a result, if we want to guarantee that the average linear swap regret incurred by the learner is at most $\delta > 0$, it suffices to run $T = O(\frac{d^2}{\delta^2})$ iterations each with precision $\epsilon = O( \frac{\delta}{\sqrt{d}} )$.

As discussed in the main body, the problem of approximate minimax equilibrium computation is amenable to (external) no-regret learning~\citep{Freund97:decision}. The best algorithm for this problem depends on the desired precision. For example, the sublinear algorithm of~\citet{Grigoriadis95:Sublinear} computes an $\epsilon$-approximate minimax equilibrium in time $\Tilde{O}(d/\epsilon^2)$ in the $\ell_1$-$\ell_1$ geometry. In a moderately higher precision regime, improved algorithms are known~\citep{Rakhlin13:Optimization,Carmon19:Variance,Carmon20:Coordinate}. Interestingly, the underlying sequence of minimax games to be solved changes slowly; it is possible that this can be leveraged to obtain further algorithmic improvements.

From a more practical standpoint, reducing linear swap regret to a sequence of minimax problems offers a significant advantage. Recent progress in practical zero-sum game solving can now be leveraged for the problem of minimizing linear swap regret.

\begin{algorithm}[t]
\caption{Approximate response-based approachability}
\label{alg:shimkin-approx}
\DontPrintSemicolon
\setcounter{AlgoLine}{0}
\KwIn{Horizon $T$, convex sets $\cP, \cL \subset \R^d$, best response map $b: \cL \to \cP$}
\KwOut{Sequence of strategies $p_1, \dots, p_T$}

Initialize $\bar{U}_0 \gets 0 \in \R^{d \times (d+1)}$\;

\For{$t=1$ \KwTo $T$}{
    Compute a pair of $\epsilon_t$-approximate minimax strategies $(p_t, \ell^*_t)$ of the (bilinear) zero-sum game
    \[
    \min_{p \in \cP} \max_{\ell \in \cL} \inp{ \bar{U}_{t-1}, (\ell \otimes p, \ell)},
    \]
    
    Set $s_t \gets (\ell_t^* \otimes b(\ell_t^*), \ell_t^*)$\;
    
    \textbf{Play} the strategy $p_t$ and \textbf{observe} the loss $\ell_t$\;
    
    Set $\kappa_t \gets (\ell_t \otimes p_t, \ell_t)$\;

    Update $U_t \gets U_{t-1} + (\kappa_t - s_t)$ and $\bar{U}_t = \frac{1}{t} U_t$ \;
}
\end{algorithm}
\section{Omitted proofs}
\label{sec:proofs}

This section contains the proofs omitted from the main body. We begin with~\Cref{lemma:AL-red}.

\ALred*

\begin{proof}
    Let $s^* = \arg \min_{s \in \cS} \| \bk_T-s \|_{F}$. By definition, $s^*$ can be decomposed as $s^* = \sum_{\ell \in \cL} w_\ell (\ell \otimes b(\ell), \ell)$. For any affine endomorphism $\phi = (M,a) \in \affend(\cP)$,
    \begin{align}
        \langle (I_d,0) - (A,c), s^* \rangle &= \sum_{\ell \in \cL} w_\ell \left( \langle (I_d,0), (\ell \otimes b(\ell), \ell) \rangle - \langle (M,a), (\ell \otimes b(\ell), \ell) \rangle \right) \notag \\
        &= \sum_{\ell \in \cL} w_\ell \left( \langle \ell, b(\ell) \rangle - (\langle \ell, M b(\ell) \rangle + \langle \ell, a \rangle) \right) \notag \\
        &= \sum_{\ell \in \cL} w_\ell \Big( \inp{\ell, b(\ell)} - \inp{\ell, \phi(b(\ell))} \Big) \leq 0 \label{align:s*}
    \end{align}
    since $b(\ell)$ minimizes $\langle \ell, p \rangle$ (by definition) and $\phi(b(\ell)) \in \cP$ ($\phi \in \affend(\cP)$). As a result, using~\eqref{align:s*},
    \begin{align*}
        \LSR_T &= T \max_{(M,a) \in \affend(\cP)} \langle (I_d,0)-(M,a), \bk_T \rangle \\
        &\leq T \max_{(M,a) \in \affend(\cP)} \langle (I_d,0)-(M,a), \bk_T - s^* \rangle \\
        &\leq T \max_{(M,a) \in \affend(\cP)} \| (I_d,0)-(M,a) \|_{F} \| \bk_T - s^* \|_F,
    \end{align*}
    where the last inequality is Cauchy-Schwarz. This completes the proof since $\| (I_d, 0) - (M, a) \|_F \leq \|(I_d, 0) \|_F + \|(M, a) \|_F$ and $(I_d, 0) \in \affend(\cP)$.
\end{proof}

The correctness of~\Cref{alg:shimkin} hinges on the Pythagorean lemma (\Cref{lemma:Pythagorean}). The main precondition of that lemma is satisfied by virtue of the minimax theorem, as formalized below.

\Shimkininvariant*

\begin{proof}
    Let $g_t(p, \ell) \defeq \inp{U_{t-1}, (\ell \otimes p, \ell)}$ be the payoff function of the auxiliary minimax game at step $t$, as defined in~\Cref{alg:shimkin}. Since $\cP$ and $\cL$ are both convex and compact and $g$ is bilinear, the minimax theorem applies~\citep{vonNeumann28:Zur}. We let $V_t$ be the value of the game,
    \[
        V_t \defeq \min_{p \in \cP} \max_{\ell \in \cL} g_t(p, \ell) = \max_{\ell \in \cL} \min_{p \in \cP} g_t(p, \ell).
    \]
    By the definition of the learner's strategy $p_t \in \arg\min_{p \in \cP} \max_{\ell \in \cL} g_t(p, \ell)$, we have that for any (realized) loss $\ell_t \in \cL$,
    \begin{equation}
        \label{eq:game-upper}
        \inp{U_{t-1}, \kappa_t} = g_t(p_t, \ell_t) \leq V_t.
    \end{equation}
    Furthermore, since $\ell_t^* \in \arg\max_{\ell \in \cL} \min_{p \in \cP} g_t(p, \ell)$,
    \[
        V_t = \min_{p \in \cP} g_t(p, \ell_t^*).
    \]
    Now, the target point is chosen as $s_t = (\ell_t^* \otimes b(\ell_t^*), \ell_t^*) \in \cS$. We can thus obtain the upper bound
    \begin{equation}
        \label{eq:game-lower}
        V_t = \min_{p \in \cP} \inp{U_{t-1}, (\ell_t^* \otimes p, \ell_t^*)} \leq \inp{U_{t-1}, (\ell_t^* \otimes b(\ell_t^*), \ell_t^*)} = \inp{U_{t-1}, s_t}.
    \end{equation}
    Combining \cref{eq:game-upper} and \cref{eq:game-lower},
    \[
        \inp{U_{t-1}, \kappa_t} \leq V_t \leq \inp{U_{t-1}, s_t} \implies \inp{U_{t-1}, \kappa_t - s_t} \leq 0.
    \]
\end{proof}

We continue by establishing a basic invariance that allows suitably transforming the strategy and loss sets; this is essential for our preconditioning approach.

\invariance*

\begin{proof}
    First, we observe that the instantaneous loss is preserved. That is, for any $t$, $\langle \ell'_t, p'_t \rangle = \langle A^{-\top} \ell_t, A p_t \rangle = \langle \ell_t, p_t \rangle$.
    
    Next, we establish a correspondence between $\affend(\cP)$ and $\affend(\cP')$. We consider the mapping $\Psi: \affend(\cP) \to \affend(\cP')$ defined by $\Psi(\phi) = A \circ \phi \circ A^{-1}$. For any $\phi \in \affend(\cP)$ and any $p' \in \cP'$, we have $A^{-1} p' \in \cP$, so $\phi(A^{-1} p') \in \cP$, which implies $\Psi(\phi)(p') = A(\phi(A^{-1} p')) \in \cP'$. Thus, $\Psi$ maps into $\affend(\cP')$. Since $A$ is invertible, this mapping is a bijection; for any $\phi' \in \affend(\cP')$, the inverse mapping is given by $\phi = A^{-1} \circ \phi' \circ A$.
    
    Now, for any $\phi \in \affend(\cP)$ and its corresponding $\phi' = \Psi(\phi) \in \affend(\cP')$, we have
    \begin{align*}
        \sum_{t=1}^T \langle \ell'_t, \phi'(p'_t) \rangle 
        = \sum_{t=1}^T \langle A^{-\top} \ell_t, (A \circ \phi \circ A^{-1})(A p_t) \rangle = \sum_{t=1}^T \langle A^{-\top} \ell_t, A \phi(p_t) \rangle =\sum_{t=1}^T \langle \ell_t, \phi(p_t) \rangle.
    \end{align*}
    This implies that $\LSR_T = \LSR_T'$, completing the proof.
\end{proof}

We next bound the Frobenius norm of endomorphisms. We first point out that the containment $B_2 \subseteq \cP \subseteq \sqrt{d} B_2$ guaranteed by John's theorem (\Cref{theorem:John}) does not suffice on its own.

\lopsided*

\begin{proof}
We define $\cP$ to be an ellipsoid. Specifically, we choose axis lengths
\[
s_1=\cdots=s_k=\sqrt d, \quad s_{k+1}=\cdots=s_d=1,
\]
where \(k=\lfloor d/2 \rfloor\), so that
\[
\cP \defeq \left\{ p \in \mathbb{R}^d : 
\sum_{i=1}^d \frac{p_i^2}{s_i^2} \le 1
\right\}.
\]
Since \(s_i\ge1\) for each $i$, any point \(p \in \R^d \) such that \(\|p\|_2\le1\) satisfies
\(\sum_{i=1}^d p_i^2/s_i^2 \le 1\), so \(B_2 \subseteq \cP\). Furthermore, since $s_i \leq \sqrt{d}$ for each $i$, it follows that \(\cP \subseteq \sqrt{d} B_2\).

Now, let \(D=\mathrm{diag}(s_1,\dots,s_d)\), so that $\cP = \{ p: \| D^{-1} p \| \leq 1 \}$. For an orthogonal matrix \(U\), we define
\[
M = D U D^{-1}.
\]
Since $D^{-1} M D = U$ is orthogonal, we have that for any $p \in \cP$,
\[
\|D^{-1} M  p \|_2 
= \| U (D^{-1} p) \|_2
= \|D^{-1} p\|_2
\le 1.
\]
Thus, \(M(\cP)\subseteq \cP \), so $M$ is an endomorphism on $\cP$.

To make \(\|M \|_F\) large, we choose \(U\) to be a permutation matrix that maps
the \(k\) indices with \(s_j=\sqrt d\) into distinct indices with \(s_i=1\). We can then write
\[
M_{i, j} = \frac{s_i}{s_j} U_{i, j},
\]
so each column has exactly one nonzero entry. It then follows that $\|M\|_F^2 \geq k (\sqrt{d})^2 = d \lfloor d/2 \rfloor$.
\end{proof}

For this reason, our next proof makes crucial use of John's decomposition.

\Frobbound*

\begin{proof}
    John's theorem (\Cref{theorem:John}) guarantees the existence of a finite set of points $\xi_1, \dots, \xi_m \in \partial \cP \cap \partial B_2$ and weights $c_1, \dots, c_m > 0$ such that $\sum_{i=1}^m c_i \xi_i \otimes \xi_i = I_d$ and $\sum_{i=1}^m c_i \xi_i = 0$. Since $\|\xi_i \|_2 = 1$ for each $i \in [m]$,
    \[
        d = \tr(I_d) = \tr \left( \sum_{i=1}^m c_i \xi_i \otimes \xi_i \right) = \sum_{i=1}^m c_i \tr( \xi_i \otimes \xi_i) = \sum_{i=1}^m c_i \|\xi_i\|_2^2 = \sum_{i=1}^m c_i.
    \]
    For each contact point $\xi_i$, it holds that $\langle p, \xi_i \rangle \leq 1$ for any $p \in \cP$. In particular, since $\phi(\cP) \subseteq \cP$,
    \[
        \max_{p \in \cP} \langle M p + a, \xi_i \rangle \leq 1.
    \]
    Thus, since $p = M^\top \xi_i/ \| M^\top \xi_i \|_2 \in \cP$ (by the fact that $B_2 \subseteq \cP$),
    \[
        \|M^\top \xi_i \|_2 \leq \max_{p \in \cP} \langle p, M^\top \xi_i \rangle \leq 1 - \langle a, \xi_i \rangle.
    \]
    Finally,
    \begin{align*}
        \|M \|_F^2 = \tr( M^\top I_d M ) = \tr \left( M^\top \left( \sum_{i=1}^m c_i \xi_i \otimes \xi_i \right) M \right) &= \sum_{i=1}^m c_i \tr( M^\top \xi_i \xi_i^\top M ) \\
        &= \sum_{i=1}^m c_i \| M^\top \xi_i \|_2^2.
    \end{align*}
    Using the bound $\|M^\top \xi_i \|_2 \leq 1 - \langle a, \xi_i \rangle $,
    \begin{align*}
        \|M \|_F^2 &\leq \sum_{i=1}^m c_i ( 1 - \langle a, \xi_i \rangle )^2 = \sum_{i=1}^m c_i - 2 \sum_{i=1}^m c_i \langle a, \xi_i \rangle + \sum_{i=1}^m c_i \langle a, \xi_i \rangle^2 \\
        &= d - 2 \left\langle a, \sum_{i=1}^m c_i \xi_i \right\rangle + a^\top \left( \sum_{i=1}^m c_i \xi_i \otimes \xi_i \right) a \\
        &= d + \| a \|_2^2.
    \end{align*}
    When $\cP$ is centrally symmetric, it follows that $M \cdot 0 + a \in \cP$ since $0 \in \cP$, so $\|a\|_2 \leq \sqrt{d}$ by~\Cref{theorem:John}.
\end{proof}

\symmetrization*

\begin{proof}
    It suffices to show that $\psi(\cP) \subseteq \conv ( \cP, - \cP) $ and $\psi(-\cP) \subseteq \conv (\cP, -\cP)$. For any point $p \in \cP$, $\psi(p) = \frac{1}{3} \phi(p) \in \frac{1}{3} \cP \in \conv ( \cP, - \cP)$. Moreover, $\psi(-p) = - \frac{1}{3} M p + \frac{1}{3} a = - \frac{1}{3} (M p + a) + \frac{2}{3} a = \frac{1}{3} (- \phi(p)) + \frac{2}{3} \phi(0) \in \conv( \cP, -\cP)$.
\end{proof}

To complete the proof of~\Cref{thm:main-ub}, it suffices to establish the following bound.

\cKbound*

\begin{proof}
    By assumption, it follows that $\cL$ is the polar set of $\conv(\cP, - \cP)$. By~\Cref{theorem:John}, since $\conv(\cP, -\cP)$ is in John's position, we have $\|p \|_2 \leq \sqrt{d}$ for any $p \in \cP$. Furthermore, since $\conv(\cP,- \cP) \supseteq B_2$, it follows that $\cL \subseteq B_2^\circ = B_2$. We conclude that $\|\ell \|_2 \leq 1$, and the claim follows.
\end{proof}

%Similar reasoning yields a weaker upper bound without the assumption that $\cP$ is centrally symmetric, as we state below.

%\begin{lemma}
%    \label{lemma:weak-bound}
%    If $\cP$ is a convex body in John's position and $\cL$ is such that $\langle p, \ell \rangle \leq 1$ for all $p \in \cP$ and $\ell \in \cL$, we have $\max_{p \in \cP, \ell \in \cL} \| \ell \otimes p \|_F \leq d$.
%\end{lemma}

We continue with the proofs from~\Cref{sec:swappoly}, beginning with~\Cref{lemma:poly-AL-red}.

\polyALred*

\begin{proof}
    Let $s^* = \arg \min_{s \in \cS} \| \bk_T-s\|_{F}$. By the definition of the target set $\cS$, there exist weights $w_\ell \geq 0, \sum_{\ell \in \cL} w_\ell= 1$ such that $s^* = \sum_{\ell \in \cL} w_\ell (\ell \otimes B(\ell))$. For any function $M \in \Phi^m $,
    \begin{align*}
        \inp{J_d - M,s^*} &= \sum_{\ell \in \cL} w_\ell \left(\inp{\ell,b(\ell)}-\inp{\ell, M B(\ell)} \right) \leq 0
    \end{align*}
    since $b(\ell) = \arg\min_{p \in \cP} \inp{\ell,p}$, $J_d B(\ell) = J_d m(b(\ell)) = b(\ell)$, and $M B(\ell) \in \cP$. As a result,
    \begin{align*}
        \PDSR_T &= T \max_{M \in \Phi^m} \inp{J_d - M, \bk_T} \\ 
        &\leq T \max_{M \in \Phi^m }\inp{J_d - M, \bk_T-s^*}\\
        &\leq T \max_{M \in \Phi^m} \| J_d -M \|_F \| \bk_T-s^* \|_F.
    \end{align*}
    Given that $\| J_d - M \|_F \leq \| J_d \|_F + \|M \|_F$ and $J_d \in \Phi^m$, the claim follows.
\end{proof}

The proof of the following lemma is similar to that of~\Cref{lemma:invariance}, making crucial use of the minimax theorem.

\shimkinmixed*

\begin{proof}
    Let $g_t(\mu, \ell) \defeq \E_{p \sim \mu} \ell^\top U_{t-1} m(p_t)$ be the payoff function of the auxiliary zero-sum game. By the minimax theorem,
    \[
        V_t \defeq \min_{\mu \in \Delta(\cP)} \max_{\ell \in \cL} g_t(\mu, \ell) = \max_{\ell \in \cL} \min_{\mu \in \Delta(\cP)} g_t(\mu, \ell).
    \]
    By the definition of the learner's strategy $\mu_t \in \arg\min_{\mu} \max_{\ell} g_t(\mu, \ell)$, we have that for any (realized) loss $\ell_t \in \cL$,
    \begin{equation}
        \label{eq:game-upper-mixed}
        \E_{p_t \sim \mu_t} \ell_t^\top U_{t-1} m(p_t)  = g_t(\mu_t, \ell_t) \leq V_t.
    \end{equation}
    Moreover,
    \begin{equation}
        \label{eq:game-lower-mixed}
        V_t = \min_{\mu \in \Delta(\cP)} \E_{p \sim \mu} \langle U_{t-1}, \ell_t^* \otimes m(p_t) \rangle \leq \inp{U_{t-1}, \ell_t^* \otimes m(b(\ell_t^*))} = \inp{U_{t-1}, s_t}.
    \end{equation}
    Combining \cref{eq:game-upper-mixed} and \cref{eq:game-lower-mixed},
    \[
        \langle U_{t-1}, \kappa_t \rangle  \leq V_t \leq \inp{U_{t-1}, s_t}.
    \]
\end{proof}

We conclude with the proofs from~\Cref{sec:lowerbound}, leading to the proof of~\Cref{thm:lowerbound}.

\movement*

\begin{proof}
We consider the following adaptive adversary. We divide $T$ into $d$ periods, each comprising $T/d$ rounds; we can assume without any loss of generality that $T$ is divisible by $d$. In the first period, the adversary consistently selects the loss $y_t = (-1, -\frac{1}{2} , \dots, -\frac{1}{2} )$. At the end of each period, the adversary calculates $\sum_{\tau=1}^t \langle y_\tau, x_\tau - ( A x_\tau + x^*) \rangle$, which is the cumulative linear swap regret up to that time $t$. If it exceeds $d \sqrt{T}$, it proceeds by selecting the loss $y_t = (0, 0, \dots, 0)$ thereafter; we call this scenario \emph{termination}. This is the only place where our adversary is adaptive. By construction, termination always implies that $\sum_{t=1}^T \langle y_t, x_t - ( A x_t + x^*) \rangle \geq d \sqrt{T}$. Assuming there is no termination, in the second period, the adversary selects the loss $y_t = (0, -1, - \frac{1}{2}, \dots, - \frac{1}{2}) $. Subject to not experiencing termination, the loss in the last period is set to $(0, 0, \dots, 0, -1)$. This sequence of losses is designed to entice the learner to play each vertex of $\Delta^d$ with high probability, in the sense of~\Cref{lemma:movement}.

Let us assume that there was no termination. The first observation is that there is a fixed hindsight deviation $x^* \in \Delta^d$ that secures (at most) $-\frac{1}{2} t$ cumulative loss after $t$ rounds, namely $x^*_{d} = 1$. As a result, if $t \in \{ T/d, 2T/d, \dots, T\}$, corresponding to the end of a period, we have
\[
    \sum_{\tau = 1}^t \langle x_\tau, y_\tau \rangle \leq - \frac{1}{2} t + d \sqrt{T}.
\]
Otherwise, the learner would experience higher than $d \sqrt{T}$ linear swap regret, which would result in termination. We use $\mathcal{T}_k$ to denote all rounds in the $k$th period. By the definition of the sequence of losses $y_1, \dots, y_T$, we have that if $t = T k/d$,
\begin{equation}
    \label{eq:externaldev}
    \sum_{\tau = 1}^t \langle x_\tau, y_\tau \rangle = \sum_{\kappa=1}^k \sum_{\tau \in \mathcal{T}_\kappa} \left( - x_{\tau, \kappa} - \frac{1}{2} \sum_{i > \kappa} x_{\tau, i} \right) \leq -\frac{t}{2} + d \sqrt{T}.
\end{equation}
Furthermore, again for $t = T k /d$, we consider the linear swap deviation
\[
    A: (x_1, \dots, x_{k-1}, x_k, x_{k+1}, \dots, x_d) \mapsto \left( x_1, \dots, x_{k-1}, \sum_{i \geq k} x_{i}, 0, \dots, 0 \right),
\]
which clearly satisfies $A x \in B_1$ for all $x \in B_1$. Then 
\begin{equation*}
    \sum_{\tau=1}^t \langle y_\tau, x_\tau - A x_\tau \rangle = \sum_{\kappa = 1}^k \sum_{\tau \in \mathcal{T}_\kappa } \left( \sum_{i > k} y_{\tau, i} x_{\tau, i} - y_{\tau, k} \sum_{i > k} x_{\tau, i} \right).
\end{equation*}
For $\kappa < k$ and $\tau \in \mathcal{T}_\kappa$, it holds that $y_{\tau, i} = y_{\tau, i'}$ for any $i, i' \geq k$. So,
\begin{align*}
    \sum_{\tau=1}^t \langle y_\tau, x_\tau - A x_\tau \rangle = \sum_{\tau \in \mathcal{T}_k} \left( \sum_{i > k} y_{\tau, i} x_{\tau, i} - y_{\tau, k} \sum_{i > k} x_{\tau, i} \right) &= \sum_{\tau \in \mathcal{T}_k} \left( - \frac{1}{2} \sum_{i > k} x_{\tau, i} + \sum_{i > k} x_{\tau, i} \right) \\
    &= \frac{1}{2} \sum_{\tau \in \mathcal{T}_k} \sum_{i > k} x_{\tau, i}.
\end{align*}
As a result, for any $k = 1, \dots, d$ and $t = T k /d$,
\[
  \sum_{\tau \in \mathcal{T}_k} \sum_{i > k} x_{\tau, i} \leq 2 d \sqrt{T}.
\]
So, continuing from~\eqref{eq:externaldev},
\begin{equation*}
    - \sum_{\kappa = 1}^k \sum_{\tau \in \mathcal{T}_\kappa} x_{\tau, \kappa} \leq - \frac{t}{2} + \frac{1}{2} \sum_{\kappa = 1}^k \sum_{\tau \in \mathcal{T}_\kappa} \sum_{i > k} x_{\tau, i} + d \sqrt{T} \leq - \frac{t}{2} + d (k+1) \sqrt{T}.
\end{equation*}
In particular,
\begin{equation*}
    \sum_{k=1}^d \sum_{t \in \mathcal{T}_k} x_{t, k} \geq \frac{1}{2} T - d (d+1) \sqrt{T}.
\end{equation*}
When $T \geq 16 d^2(d+1)^2$,
\[
    \sum_{k=1}^d \sum_{t \in \mathcal{T}_k} x_{t, k} \geq \frac{1}{4} T.
\]
Using the obvious inequality $ x_{t, k} \leq \mathbbm{1} \{ x_{t, k} \geq \frac{1}{8} \} + \frac{1}{8}$,
\begin{equation}
    \label{eq:linear-lb}
    \sum_{k=1}^d \sum_{t \in \mathcal{T}_k} \mathbbm{1} \left\{  x_{t, k} \geq \frac{1}{8} \right\} \geq \frac{1}{8} T.
\end{equation}
We define
\[
    I \defeq \left\{ i \in [d] : \sum_{t = 1}^T \mathbbm{1} \left\{ x_{t, i} \geq \frac{1}{8} \right\} \geq \frac{T}{16 d} \right\}. 
\]
Based on this definition, we can write
\[
    \sum_{k=1}^d \sum_{t \in \mathcal{T}_k} \mathbbm{1} \left\{  x_{t, k} \geq \frac{1}{8} \right\} \leq |I| \frac{T}{d} + (d - |I|) \frac{T}{16 d}.
\]
Combining with~\eqref{eq:linear-lb}, the claim follows.
\end{proof}

\punish*

\begin{proof}
For each time $t \in [T]$, the adversary selects uniformly at random a pair $j \neq j' \sim [d]$ and sets $\ell_{t, j} = \frac{1}{2}$, $\ell_{t, j'} = - \frac{1}{2}$, and $\ell_{t, j''} = 0$ for $j'' \neq j, j'$. Then $\E[ \langle \ell_t, p_t \rangle ] = \langle \E[\ell_t], p_t \rangle = 0$. 

Furthermore, we write
\[
    - \sum_{t=1}^T \langle \ell_t, M x_t \rangle = - \sum_{i=1}^d \sum_{j=1}^d M_{j, i} \sum_{t=1}^T \left( x_{t, i} \ell_{t, j} \right).
\]
We select $M \in [-1, 1]^{d \times d}$ so that
\[
    \max_{M \in [-1, 1]^{d \times d} } - \sum_{t=1}^T \langle \ell_t, M x_t \rangle = \sum_{i=1}^d \sum_{j=1}^d \left| \sum_{t = 1}^T x_{t, i} \ell_{t, j} \right|.
\]
Specifically,
\[
    M_{j, i} = 
    \begin{cases}
        1 & \text{if } \sum_{t=1}^T ( x_{t, i} \ell_{t, j}) < 0, \\
        -1 & \text{if } \sum_{t=1}^T ( x_{t, i} \ell_{t, j}) \geq 0.
    \end{cases}
\]

For a fixed $i \in [d]$ and $j \in [d]$, let $Z_{i, j} = \sum_{t = 1}^T x_{t, i} \ell_{t, j}$. $(x_{t, i} \ell_{t, j})_{t \geq 1}$ is a conditionally symmetric martingale difference sequence with respect to the filtration $\mathcal{F}_{t-1}$ generated by the history of play. Indeed, $\E[ x_{t, i} \ell_{t, j} \mid \mathcal{F}_{t-1}] = x_{t, i} \E[ \ell_{t, j} ] = 0$. Let $(\epsilon_1, \dots, \epsilon_T)$ be a sequence of i.i.d. Radamacher random variables independent of the play. By symmetry, $Z_{i, j}$ has the same distribution as the decoupled sum $\sum_{t=1}^T \epsilon_t x_{t, i} \ell_{t, j} $. For any fixed realization of $(x_{t, i} \ell_{t, j} )_{t = 1}^T$, the Khintchine inequality (\emph{e.g.}~\citealp{Pena12:Decoupling}) implies that there exists an absolute constant $C > 0$ such that
\[
    \E_{\epsilon_1, \dots, \epsilon_T} \left[ \left| \sum_{t=1}^T \epsilon_t x_{t, i} \ell_{t, j} \right| \right] \geq C  \sqrt{ \sum_{t =1}^T x_{t, i}^2 \ell_{t, j}^2 }.
\]
As a result, taking the expectation over all realizations of $(x_{t, i} \ell_{t, j} )_{t = 1}^T$,
\[
    \E[ |Z_{i, j}|] \geq C \E\left[ \sqrt{ \sum_{t =1}^T x_{t, i}^2 \ell_{t, j}^2 } \right].
\]
Now, by definition of the loss, we have $\E[ \ell_{t, j}^2 ] = \frac{2}{d} \frac{1}{4} = \frac{1}{2d}$. For a given set $\mathcal{T}_i \subseteq [T]$ with $|\mathcal{T}_i| \geq T/(16 d)$, a Chernoff bound yields
\[
    \mathbb{P} \left( \sum_{t \in \mathcal{T}_i} \ell_{t, j}^2 \leq \frac{1}{2} \frac{| \mathcal{T}_i|}{2d} \right) \leq \exp \left( - \frac{ | \mathcal{T}_i | }{16 d} \right) \leq \exp \left( - \frac{T}{256 d^2} \right) \leq \exp \left( - \frac{(d+1)^2}{16} \right) \leq \frac{1}{2},
\]
where we used the fact that $|\mathcal{T}_i| \geq T / (16 d)$, $T \geq 16 d^2 (d+1)^2$, and $d \geq 3$. This in turn implies 
\[
  \E\left[ \sqrt{\sum_{t \in \mathcal{T}_i} \ell^2_{t, j}} \right] \geq \sqrt{ \frac{| \mathcal{T}_i|}{4 d}} \mathbb{P} \left( \sum_{t \in \mathcal{T}_i} \ell_{t, j}^2 \geq \frac{ | \mathcal{T}_i | }{4 d} \right) \geq \frac{1}{4} \sqrt{\frac{ |\mathcal{T}_i| }{d}}.  
\]
As a result, for any $i \in I$,
\[
    \E[ |Z_{i, j}| ] \geq \frac{C}{8} \left(  \frac{1}{d} \sqrt{T} \right).
\]
Combining, we have
\[
    \max_{M \in [-1, 1]^{d \times d} } - \sum_{t=1}^T \langle \ell_t, M x_t \rangle \geq \sum_{i \in I} \sum_{j = 1}^d \E[ |Z_{i, j}| ] \geq  \Omega(d \sqrt{T}).
\]
\end{proof}

Combining~\Cref{lemma:movement} and~\Cref{lemma:punishment}, \Cref{thm:lowerbound} follows.

\end{document}